\documentclass[11pt]{article}

\usepackage[margin=1in]{geometry}
\usepackage{times}
\usepackage{microtype}
\usepackage{url}
\usepackage[hidelinks]{hyperref}

\usepackage{amsmath,amssymb,amsfonts}
\usepackage{amsthm}
\usepackage{mathtools}

\usepackage{graphicx}
\usepackage{subcaption}
\usepackage{booktabs}
\usepackage{multirow}

\usepackage{algorithm}
\usepackage{algpseudocode}

\usepackage{enumitem}

\theoremstyle{plain}

\newtheorem{lemma}{Lemma}
\newtheorem{proposition}{Proposition}

\theoremstyle{definition}

\theoremstyle{remark}

\usepackage{tikz}
\usetikzlibrary{arrows.meta,positioning,shapes.geometric}

\title{\bfseries Adaptive Learning Guided by Bias-Noise-Alignment Diagnostics}
\author{
Akash Samanta\thanks{Corresponding author: Akash Samanta (\texttt{akash.samanta@ontariotechu.net}).} \
and Sheldon Williamson
\\
Department of Electrical, Computer and Software Engineering, Ontario Tech University, Oshawa, Canada\\
}

\date{}

\begin{document}
\maketitle

% ---------- abstract ----------

\begin{abstract}
Learning systems deployed in nonstationary and safety-critical environments often suffer from instability, slow convergence, or brittle adaptation when learning dynamics evolve over time. While modern optimization, reinforcement learning, and meta-learning methods adapt to gradient statistics, they largely ignore the temporal structure of the error signal itself. This paper proposes a diagnostic-driven adaptive learning framework that explicitly models error evolution through a principled decomposition into \emph{bias}, capturing persistent drift; \emph{noise}, capturing stochastic variability; and \emph{alignment}, capturing repeated directional excitation leading to overshoot. These diagnostics are computed online from lightweight statistics of loss or temporal-difference (TD) error trajectories and are independent of model architecture or task domain. We show that the proposed bias-noise-alignment decomposition provides a unifying control backbone for supervised optimization, actor-critic reinforcement learning, and learned optimizers. Within this framework, we introduce three diagnostic-driven instantiations: the Human-inspired Supervised Adaptive Optimizer (HSAO), Hybrid Error-Diagnostic Reinforcement Learning (HED-RL) for actor-critic methods, and the Meta-Learned Learning Policy (MLLP). Under standard smoothness assumptions, we establish bounded effective updates and stability properties for all cases. Representative diagnostic illustrations in actor-critic learning highlight how the proposed signals modulate adaptation in response to TD error structure. Overall, this work elevates error evolution to a first-class object in adaptive learning and provides an interpretable, lightweight foundation for reliable learning in dynamic environments.
\end{abstract}

% ---------- keywords ----------

% =========================================================
\section{Introduction}
Modern learning systems are increasingly deployed in environments that are nonstationary, noisy, and safety-critical, where the statistical properties of data and objectives evolve over time. In such settings, learning dynamics are often dominated not only by stochastic gradients but also by structured changes in operating conditions, feedback reliability, and temporal dependencies. Examples include sequential prediction under regime shifts, control under uncertain dynamics, and continual or task-adaptive learning. Despite remarkable
progress in optimization and reinforcement learning (RL), ensuring stable, interpretable, and reliable adaptation under these conditions remains a fundamental challenge.

Most widely used optimization algorithms adapt learning rates based on instantaneous or exponentially averaged gradient statistics \cite{rubio2017convergence}. Methods such as Adam and AdamW \cite{kingma2014adam,loshchilov2017decoupled, zaheer2018adaptive}, large-batch optimizers such as LAMB \cite{you2019large}, sign-based methods such as Lion~\cite{chen2023lion}, and second-order approximations such as Sophia~\cite{liu2023sophia} have proven effective across a wide range of applications. However, these methods primarily react to gradient magnitude and variance, without explicitly accounting for how the \emph{error signal itself evolves} during training. In nonstationary settings, this limitation can lead to unstable convergence, oscillatory behavior, or overly conservative adaptation, particularly in long-horizon sequence models and control-oriented learning problems. Related challenges arise in RL, where training stability is strongly influenced by the structure of temporal-difference (TD) errors. Stabilization techniques such as trust-region methods~\cite{schulman2015trust}, proximal policy optimization (PPO)~\cite{schulman2017proximal}, entropy regularization, and advantage estimation reduce variance and catastrophic updates, but they do not explicitly distinguish between systematic error drift and stochastic
fluctuations in the learning signal. As a result, policy learning remains highly sensitive to hyperparameters, noise, and implementation details, particularly
under changing dynamics or partial observability~\cite{henderson2018deep}.

Meta-learning and learned optimizers provide another avenue for adaptation by learning update rules or initializations that generalize across tasks. Approaches such as Model-agnostic meta-learning (MAML)~\cite{finn2017model}, Meta-SGD~\cite{li2017meta}, Reptile, and recurrent learned optimizers~\cite{andrychowicz2016learning} demonstrate that optimization itself can be learned from data. While powerful, these methods typically condition on gradients and internal hidden states, offering limited interpretability of \emph{why} or \emph{when} adaptation occurs. Moreover, without explicit mechanisms to assess the reliability of the error signal, learned optimizers may still exhibit brittle behavior under unseen or highly nonstationary conditions. A common thread across supervised learning, RL, and meta-learning is that adaptation decisions are largely driven by gradient-based signals, while the temporal structure of the error signal is treated as incidental. In contrast, biological and human motor learning exhibit a qualitatively different behavior, adaptation is modulated based on perceived error trends, variability, and repeated overshoot rather than instantaneous feedback alone. Persistent drift triggers corrective action, high variability leads to cautious updates, and repeated overshoot prompts attenuation of movement amplitude. These observations suggest that \emph{error evolution} contains interpretable structure that can serve as a principled control signal for adaptive learning.

Motivated by this perspective, we propose a diagnostic-driven adaptive
learning framework that explicitly models error dynamics through a
simple yet expressive decomposition into three components:
\emph{bias}, capturing persistent error drift;
\emph{noise}, capturing stochastic variability around the trend; and
\emph{alignment}, capturing systematic overshoot arising from repeated
update directions. These diagnostics are computed online from
lightweight statistics of loss or TD-error evolution and are independent of model architecture or task domain. Rather than replacing gradient-based learning, the proposed diagnostics complement gradients by regulating update magnitude, direction, and exploration based on the reliability and structure of the learning signal. The central contribution of this work is to demonstrate that a bias-noise-alignment decomposition of error dynamics provides a unifying backbone for adaptive learning across multiple paradigms. We show that the same diagnostic signals can be systematically reused to stabilize supervised optimization, regulate actor-critic RL, and condition learned optimizers for rapid task adaptation. By explicitly elevating error evolution to a first-class object, this framework bridges ideas from optimization, control, and meta-learning within a single interpretable perspective. Under standard smoothness assumptions, we establish bounded effective updates and descent-style stability guarantees for diagnostic-driven adaptation. The proposed approach does not replace existing optimizers or control algorithms; rather, it enhances their robustness when learning dynamics become
biased, noisy, or nonstationary. The major contributions of this work are summarized below. 
\begin{itemize}
  \item We introduce a lightweight \emph{bias-noise-alignment}
  decomposition of error dynamics that is model-agnostic and computable
  online from loss or TD-error trajectories.
  \item We instantiate a single diagnostic framework across supervised   learning, actor-critic RL, and meta-learned optimizers, demonstrating its role as a unifying mechanism for adaptive learning.
  \item We establish theoretical stability properties, including bounded effective updates and descent-style behavior under standard smoothness assumptions.
  \item We provide representative diagnostic illustrations showing how diagnostic-driven adaptation modulates learning dynamics to promote stability and robustness in nonstationary settings.
\end{itemize}

\section{Review of Literature}

This work connects adaptive optimization, RL stabilization, and meta-learning. We briefly review relevant literature
in these areas and highlight the limitations that motivate a diagnostic-driven perspective on adaptive learning.

\subsection{Adaptive Optimization Methods}
Stochastic gradient-based optimization remains the backbone of modern machine learning. Classical methods such as SGD with momentum adapt update directions using first-order information, while adaptive methods such as Adam and AdamW~\cite{kingma2014adam,loshchilov2017decoupled} further
normalize updates using gradient moments. These approaches provide robustness to scale differences across parameters and have become standard in training deep networks \cite{ioannou2023adalip}. Extensions such as AdaFactor~\cite{shazeer2018adafactor}, LAMB~\cite{you2019large}, and Lion~\cite{chen2023lion} target improved scalability, memory efficiency, or stability in large-batch and
large-model regimes. More recently, sharpness-aware optimization methods such as SAM~\cite{foret2020sharpness} and its variants have been proposed to improve generalization by explicitly seeking flat minima. While these methods introduce an implicit sensitivity to local loss geometry, they typically rely on inner maximization steps or gradient perturbations and do not explicitly model the temporal evolution of the error
signal. Second-order or quasi-second-order approaches, including Sophia~\cite{liu2023sophia}, approximate curvature information to improve convergence speed, but remain primarily reactive to instantaneous gradient statistics. A common limitation across these methods is that adaptation decisions are driven by gradient magnitude, variance, or curvature proxies, rather than by structured information about how the error evolves over time. As a result, these optimizers do not explicitly distinguish between persistent error drift, stochastic fluctuations, or repeated overshoot, which can lead to instability or overly conservative updates under nonstationary learning conditions.

\subsection{Stability in Reinforcement Learning}
RL introduces additional challenges due to delayed rewards, bootstrapping, and nonstationary data distributions induced by policy
updates ~\cite{dulac2021challenges}. actor-critic methods mitigate variance by learning a value function, while advantage estimation techniques further stabilize gradient estimates.
Trust-region approaches such as TRPO~\cite{schulman2015trust} constrain policy
updates to avoid destructive changes, and  PPO~\cite{schulman2017proximal} provides a computationally efficient approximation through clipped objectives. Entropy regularization is commonly used to encourage exploration and prevent premature convergence. However, entropy coefficients are typically fixed or scheduled heuristically, without explicit consideration of the structure of the TD error. Empirical studies have shown that deep RL methods remain highly sensitive to hyperparameters, random seeds, and implementation details~\cite{henderson2018deep}, particularly under noisy or changing dynamics. While existing stabilization techniques limit update magnitude or variance, they do not explicitly distinguish whether instability arises from systematic bias in TD errors or from stochastic noise. Consequently, learning algorithms may respond uniformly to fundamentally different error regimes, motivating the need for diagnostic signals that can disentangle these effects.

\subsection{Meta-Learning and Learned Optimizers}
Meta-learning aims to enable rapid adaptation across tasks by learning how to learn. MAML~\cite{finn2017model} focuses on learning an initialization that supports fast task-specific adaptation using standard gradient descent. Extensions such as Meta-stochastic gradient descent (Meta-SGD)~\cite{li2017meta} jointly learn initialization and per-parameter learning rates, improving flexibility and
adaptation speed. An alternative line of work treats the optimization process itself as a learning problem \cite{chen2022learning}. Learned optimizers parameterize update rules using neural networks and are trained through meta-gradient descent across tasks~\cite{andrychowicz2016learning}. Recurrent and coordinate-wise learned optimizers have demonstrated impressive data efficiency in few-shot and continual learning settings. Despite these advances, learned optimizers often rely on opaque internal states and gradient inputs, making it difficult to interpret or reason about their adaptation behavior. Moreover, without explicit mechanisms to assess the reliability of the learning signal, learned optimizers may generalize poorly to tasks with substantially different noise characteristics or nonstationary dynamics. This limitation suggests that conditioning learned updates on interpretable diagnostics of error structure may improve robustness and transparency.

\subsection{Human-Inspired and Control-Theoretic Perspectives}

A broad class of learning algorithms has been motivated by principles drawn from human cognition, motor learning, and classical control theory. Representative examples include gain scheduling and adaptive step-size control in optimization, uncertainty-aware exploration and entropy modulation in RL, and curriculum or experience-driven adaptation in meta-learning. These approaches have demonstrated empirical benefits in stabilizing training and improving convergence in complex, nonstationary environments. Despite their success, most existing human-inspired methods incorporate adaptivity implicitly or heuristically. In practice, adaptation is often driven by gradient magnitude, loss curvature, or hand-designed schedules, without an explicit characterization of the reason behind the unstable learning. From a control-theoretic perspective, however, instability is rarely monolithic. Persistent bias such as systematic drift, stochastic noise such as measurement or process variability, and repeated directional excitation such as overshoot or resonance are treated as distinct phenomena, each requiring different compensatory mechanisms. When these effects are conflated, corrective actions may be either overly aggressive or excessively conservative. A similar limitation is observed across adaptive optimization, RL stabilization, and meta-learning. Although these methods provide powerful tools for learning in high-dimensional and uncertain settings, they remain largely gradient-centric and do not explicitly model the temporal structure of error evolution. In particular, they do not distinguish between systematic and stochastic sources of instability, nor do they provide an interpretable signal for modulating update magnitude, direction, or exploration based on error reliability.

The present work addresses this gap by introducing a diagnostic decomposition of error dynamics into \emph{bias}, \emph{noise}, and \emph{alignment} components. Inspired by both human motor learning and control-theoretic reasoning, these diagnostics explicitly separate persistent drift, stochastic variability, and directional overshoot phenomena that humans naturally compensate for by adjusting step size, movement direction, and caution under uncertainty. Rather than serving as ad hoc heuristics, the proposed diagnostics form a unified and interpretable control signal that regulates learning across supervised optimization, actor-critic RL, and meta-learned optimizers. In this sense, the framework operationalizes human-inspired adaptation through a principled, task-agnostic representation of error evolution.

\section{Diagnostic Decomposition of Error Dynamics}
\label{sec:diagnostics}

This section discusses the diagnostic framework used in this study to explicitly model the temporal evolution of error during learning \cite{lewkowycz2020training}. Rather than relying solely on instantaneous gradients, we extract interpretable signals from the dynamics of the loss or TD error and use them as lightweight controls for adaptive learning. The proposed decomposition separates error dynamics into three complementary components namely, \emph{bias}, \emph{noise}, and \emph{alignment}. These diagnostics are general, model-agnostic, and applicable across supervised learning, RL, and meta-learning. The details of each mechanism are discussed in the section below.

\subsection{Error Evolution Signal}

Let $\ell_t$ denote a scalar learning signal evaluated at iteration $t$. In supervised learning, $\ell_t$ may correspond to a training or validation loss, while in RL it may represent a TD error or a batch-level surrogate thereof. Rather than treating $\ell_t$ as an isolated quantity, we focus on its temporal evolution. We define the incremental error signal which captures whether learning progress is improving, stagnating, or degrading over time. Unlike raw gradients, $\Delta \ell_t$ directly reflects changes in performance and aggregates the effects of optimization, stochasticity, and model mismatch.
\begin{equation}
\Delta \ell_t = \ell_t - \ell_{t-1},
\label{eq:delta_loss}
\end{equation}
 Using the incremental signal $\Delta \ell_t$ removes sensitivity to absolute loss scale and focuses the diagnostics on learning dynamics rather than task-dependent loss magnitude. To enable online computation, we track simple exponentially weighted moving averages (EMAs) of $\Delta \ell_t$, including its mean (trend), absolute magnitude, and residual variance. These statistics form the basis of the bias and noise diagnostics introduced below. Importantly, the error evolution signal does not depend on the dimensionality of the parameter space and introduces negligible computational overhead.

\subsection{Bias Diagnostic}

The \emph{bias diagnostic} is designed to capture persistent drift in the error signal. Sustained positive or negative values of $\Delta \ell_t$ indicate that updates are systematically moving the model away from or toward improved performance, respectively. We define an EMA of the error increment as
\begin{equation}
b_t = (1 - \alpha) b_{t-1} + \alpha \, \Delta \ell_t,
\label{eq:bias_ema}
\end{equation}
where $\alpha \in (0,1)$ controls the time scale of adaptation. The magnitude of $b_t$ reflects the degree of persistent error drift. To normalize the bias signal and make it scale-invariant, we define the bias ratio as
\begin{equation}
\rho^{\text{bias}}_t =
\frac{|b_t|}{\varepsilon + \nu_t},
\label{eq:bias_ratio}
\end{equation}
where $\nu_t$ is a volatility estimate defined in the next subsection and $\varepsilon > 0$ is a small constant for numerical stability. Intuitively, a large value of $\rho^{\text{bias}}_t$ indicates that the error is consistently drifting in one direction relative to its typical variation. In such regimes, aggressive updates may amplify systematic overshoot, motivating more conservative adaptation.

\subsection{Noise Diagnostic}

While bias captures long-term drift, learning dynamics are often dominated by short-term stochastic fluctuations arising from noisy gradients, partial observability, or random exploration. The \emph{noise diagnostic} is designed to quantify this variability independently of persistent trends. We first track an EMA of the absolute error increment as
\begin{equation}
\nu_t = (1 - \beta) \nu_{t-1} + \beta \, |\Delta \ell_t|,
\label{eq:volatility_ema}
\end{equation}
and an EMA of the squared residual as
\begin{equation}
\sigma_t^2 = (1 - \zeta) \sigma_{t-1}^2
+ \zeta \, (\Delta \ell_t - b_t)^2,
\label{eq:variance_ema}
\end{equation}
where $\beta, \zeta \in (0,1)$ are smoothing parameters. Further, we define the noise ratio as
\begin{equation}
\rho^{\text{noise}}_t =
\frac{\sqrt{\sigma_t^2}}{\varepsilon + |b_t|}.
\label{eq:noise_ratio}
\end{equation}

A large value of $\rho^{\text{noise}}_t$ indicates that stochastic variability dominates systematic drift, suggesting that the learning signal is unreliable. In such regimes, reducing step sizes or exploration intensity can improve stability.

\subsection{Alignment Diagnostic}
Bias and noise capture scalar properties of the error evolution, but do not account for the directional structure of parameter updates. Repeated overshoot can occur when gradients consistently align with accumulated momentum, leading to oscillations or divergence even when bias and noise are moderate. To capture this effect, we define an alignment diagnostic based on the
cosine similarity between the current gradient $g_t$ and a momentum vector
$m_t$:
\begin{equation}
s_t = (1 - \lambda) s_{t-1}
+ \lambda \,
\frac{\langle g_t, m_t \rangle}
{\|g_t\| \, \|m_t\| + \varepsilon},
\label{eq:alignment}
\end{equation}
where $\lambda \in (0,1)$ controls temporal smoothing. Unlike curvature-based measures, this alignment statistic captures repeated directional agreement over time, enabling detection of oscillatory or resonant update behavior without second-order information. The alignment score $s_t \in [-1,1]$ measures persistent agreement between the update direction and accumulated momentum. High positive alignment indicates repeated motion along a dominant direction, which can exacerbate overshoot in curved or delayed-response landscapes. Negative alignment suggests corrective behavior, while near-zero values indicate rapidly changing update directions.

\subsection{Design Principles and Interpretation}

The proposed bias-noise-alignment decomposition provides a compact and interpretable summary of learning dynamics where the bias component captures persistent error drift and signals the need to attenuate update magnitude when systematic overshoot is detected, the noise component captures stochastic variability and motivates conservative adaptation when feedback is unreliable, and the alignment component captures directional repetition and enables correction of oscillatory or resonant update behavior. Crucially, these diagnostics are orthogonal to model architecture, objective function, and learning paradigm. They rely only on scalar error signals and first-order update information, making them computationally lightweight and broadly applicable. Rather than replacing gradients or existing stabilization mechanisms, the diagnostics act as a higher-level control layer that modulates learning behavior based on the structure of the error signal. In the following sections, we show how this diagnostic backbone can be instantiated in supervised optimization, RL, and meta-learning, yielding adaptive algorithms with improved stability, interpretability, and robustness under nonstationary conditions.

% =========================================================
\section{Diagnostic-Driven Supervised Optimization}
\label{sec:hsao}

This section instantiates the proposed diagnostic framework in the context of supervised learning. We introduce the Hybrid Sharpness-Aware Optimizer (HSAO), which uses bias, noise, and alignment diagnostics to regulate learning rates and update directions under nonstationary training conditions. Consider a supervised learning problem with parameters $\theta \in \mathbb{R}^d$ and objective function
\begin{equation}
L(\theta) = \mathbb{E}_{(x,y)\sim\mathcal{D}} \left[ \ell(f_\theta(x), y) \right],
\label{eq:supervised_objective}
\end{equation}
where the data distribution $\mathcal{D}$ may be nonstationary or only partially observed. At iteration $t$, the learner receives a stochastic gradient
\begin{equation}
g_t = \nabla_\theta \ell_t(\theta_t),
\end{equation}
computed from a mini-batch or streaming data. In nonstationary settings, the effective loss landscape encountered during training can evolve over time, exhibiting sharp curvature, delayed response, or oscillatory behavior. Under such conditions, fixed learning rates or gradient-only adaptive schemes may lead to instability, slow convergence, or repeated overshoot. Our goal is therefore to design an optimizer that adapts not only to instantaneous gradient statistics, but also to the temporal structure of the observed error signal.

HSAO extends Adam-style adaptive optimization by introducing diagnostic-driven gates and a directional correction mechanism. As in Adam, we maintain first- and second-order moment estimates
\begin{align}
m_t &= \gamma m_{t-1} + (1 - \gamma) g_t, \label{eq:adam_moment1} \\
v_t &= \eta v_{t-1} + (1 - \eta) (g_t \odot g_t), \label{eq:adam_moment2}
\end{align}
with corresponding bias-corrected versions $\hat{m}_t$ and $\hat{v}_t$. To provide a stable long-term decay while allowing short-term adaptation, we employ a slowly decaying base learning-rate schedule
\begin{equation}
\bar{\alpha}_t = \frac{\alpha_0}{1 + c \log(1 + t)},
\label{eq:base_lr}
\end{equation}
which serves as a conservative reference rate. This base schedule is subsequently modulated using diagnostic information. Specifically, using the bias and noise ratios defined in Section~\ref{sec:diagnostics}, we construct multiplicative gating factors
\begin{align}
\kappa_t &= \frac{1}{1 + k_b \rho^{\text{bias}}_t}, \label{eq:bias_gate} \\
\delta_t &= \frac{1}{1 + k_n \rho^{\text{noise}}_t}, \label{eq:noise_gate}
\end{align}
where $k_b, k_n > 0$ control sensitivity to persistent drift and stochastic variability, respectively. The effective learning rate is then given by
\begin{equation}
\alpha_t^{\text{H}} = \bar{\alpha}_t \, \kappa_t \, \delta_t.
\label{eq:effective_lr}
\end{equation}
By construction, the learning rate is reduced when either sustained bias or high noise is detected, preventing aggressive updates under unreliable learning conditions. In addition to step-size modulation, HSAO introduces a conservative directional correction to mitigate repeated overshoot caused by strong alignment between gradients and accumulated momentum. The corrected gradient is defined as
\begin{equation}
\tilde{g}_t
= g_t - \tau \, s_t
\frac{\langle g_t, m_t \rangle}{\|m_t\|^2 + \varepsilon} \, m_t,
\label{eq:directional_correction}
\end{equation}
where $s_t$ denotes the alignment diagnostic, $\tau \geq 0$ controls the correction strength, and $\varepsilon$ ensures numerical stability. This term selectively attenuates update components that repeatedly reinforce overshoot-prone directions, while preserving orthogonal components that contribute to stable descent. The final HSAO parameter update is then given by
\begin{equation}
\theta_{t+1}
= \theta_t
- \alpha_t^{\text{H}}
\frac{\tilde{g}_t}{\sqrt{\hat{v}_t} + \varepsilon}.
\label{eq:hsao_update}
\end{equation}

We briefly summarize key theoretical properties of HSAO in this section. Since $\kappa_t, \delta_t \in (0,1]$ and $\bar{\alpha}_t \le \alpha_0$ by construction, the effective learning rate satisfies
\begin{equation}
0 \le \alpha_t^{\text{H}} \le \alpha_0,
\end{equation}
ensuring that diagnostic gating cannot amplify the base step size. Consequently, the update magnitude remains bounded whenever the normalized gradient is bounded. Under standard smoothness assumptions on $L(\theta)$ and bounded gradient variance, the update further satisfies a descent-style inequality
\begin{equation}
\mathbb{E}[L(\theta_{t+1})]
\le \mathbb{E}[L(\theta_t)]
- c_1 \, \mathbb{E}\!\left[\alpha_t^{\text{H}}
\left\|\frac{\tilde{g}_t}{\sqrt{\hat{v}_t}+\varepsilon}\right\|^2\right]
+ \mathcal{O}(\alpha_0^2),
\end{equation}
for some constant $c_1 > 0$, indicating that diagnostic modulation preserves descent while improving robustness to nonstationary error dynamics. Algorithm~\ref{alg:hsao} summarizes a single iteration of the proposed diagnostic-driven supervised optimizer. Before applying Adam-style moment updates \eqref{eq:adam_moment1}--\eqref{eq:adam_moment2}, HSAO updates the bias/noise EMAs \eqref{eq:bias_ema}--\eqref{eq:variance_ema} from the incremental error signal and then computes the gated step size.

\begin{algorithm}[htbp]
\caption{HSAO: Diagnostic-Driven Supervised Optimization}
\label{alg:hsao}
\begin{algorithmic}[1]
\Require Parameters $\theta_t$, gradient $g_t$, previous moments $(m_{t-1}, v_{t-1})$,
diagnostic states $(b_{t-1}, \nu_{t-1}, \sigma^2_{t-1}, s_{t-1})$
\State Update moments using \eqref{eq:adam_moment1}--\eqref{eq:adam_moment2}
\State Compute error increment $\Delta \ell_t = \ell_t - \ell_{t-1}$
\State Update bias and noise statistics using
\eqref{eq:bias_ema}--\eqref{eq:variance_ema}
\State Compute diagnostics $\rho^{\text{bias}}_t$, $\rho^{\text{noise}}_t$, $s_t$
\State Compute gated learning rate $\alpha_t^{\text{H}}$ using \eqref{eq:effective_lr}
\State Apply directional correction to obtain $\tilde{g}_t$ using \eqref{eq:directional_correction}
\State Update parameters using \eqref{eq:hsao_update}
\State \Return $\theta_{t+1}$
\end{algorithmic}
\end{algorithm}

% =========================================================
% =========================================================
\section{Diagnostic-Driven Reinforcement Learning}
\label{sec:hedrl}

We next extend the proposed diagnostic framework to RL and term the resulting method Hybrid Error-Diagnostic Reinforcement Learning (HED-RL). In RL, learning dynamics are governed by TD errors rather than supervised losses, introducing additional challenges due to bootstrapping, delayed rewards, and nonstationary data distributions induced by policy updates. We show that bias-noise-alignment diagnostics naturally generalize to TD-error signals and can be used to stabilize actor-critic learning through adaptive gating of critic updates, policy updates, and exploration. Consider a standard actor-critic setting with policy $\pi_\theta(a \mid s)$ and value function $V_\psi(s)$. The one-step TD error at time $t$ is
\begin{equation}
\delta_t = r_t + \gamma V_\psi(s_{t+1}) - V_\psi(s_t),
\label{eq:td_error}
\end{equation}
where $r_t$ is the reward and $\gamma \in (0,1)$ is the discount factor. The TD error serves as the primary learning signal for both the critic and the policy via advantage estimates. Analogous to the supervised case, we focus on the temporal structure of TD errors and track EMAs capturing bias and variability:
\begin{align}
b^{\text{TD}}_t &= (1 - \alpha) b^{\text{TD}}_{t-1} + \alpha \, \delta_t,
\label{eq:td_bias_ema} \\
\nu^{\text{TD}}_t &= (1 - \beta) \nu^{\text{TD}}_{t-1} + \beta \, |\delta_t|,
\label{eq:td_volatility_ema} \\
\sigma^{2,\text{TD}}_t
&= (1 - \zeta) \sigma^{2,\text{TD}}_{t-1}
+ \zeta \, (\delta_t - b^{\text{TD}}_t)^2,
\label{eq:td_variance_ema}
\end{align}
where $\alpha, \beta, \zeta \in (0,1)$ control smoothing. Using these statistics, we define normalized TD-error diagnostics:
\begin{align}
\rho^{\text{bias}}_t &=
\frac{|b^{\text{TD}}_t|}{\varepsilon + \nu^{\text{TD}}_t},
\label{eq:td_bias_ratio} \\
\rho^{\text{noise}}_t &=
\frac{\sqrt{\sigma^{2,\text{TD}}_t}}{\varepsilon + |b^{\text{TD}}_t|},
\label{eq:td_noise_ratio}
\end{align}
which distinguish between regimes dominated by systematic TD bias such as the model mismatch or policy-induced drift and regimes dominated by stochastic variability such as the exploration noise or partial observability. We incorporate the TD-error diagnostics into actor-critic learning by modulating effective step sizes for the critic and the policy. The value function parameters $\psi$ are updated using a gated TD step:
\begin{equation}
\psi_{t+1}
= \psi_t
+ \alpha_V
\frac{1}{1 + k_n \rho^{\text{noise}}_t}
\, \delta_t \, \nabla_\psi V_\psi(s_t),
\label{eq:critic_update}
\end{equation}
where $\alpha_V$ is the base critic learning rate and $k_n > 0$ controls sensitivity to TD noise. When stochastic variability dominates, the effective critic step size is reduced, mitigating instability caused by noisy bootstrap targets. The policy parameters $\theta$ are updated using a gated policy-gradient step:
\begin{equation}
\theta_{t+1}
= \theta_t
+ \alpha_\pi
\frac{1}{1 + k_b \rho^{\text{bias}}_t}
\, \hat{A}_t \, \nabla_\theta \log \pi_\theta(a_t \mid s_t),
\label{eq:policy_update}
\end{equation}
where $\alpha_\pi$ is the base policy learning rate and $\hat{A}_t$ is an advantage estimate. Persistent TD bias leads to smaller effective policy steps, aligning with the intuition behind trust-region methods and clipped policy updates.

Exploration is commonly encouraged through entropy regularization with a fixed or heuristically scheduled coefficient. In contrast, HED-RL adapts exploration based on TD-error diagnostics by defining an adaptive entropy coefficient
\begin{equation}
\beta_H(t)
= \beta_0
\frac{1 + \lambda_b \rho^{\text{bias}}_t}
{1 + \lambda_n \rho^{\text{noise}}_t},
\label{eq:entropy_schedule}
\end{equation}
where $\beta_0$ is a base entropy weight and $\lambda_b, \lambda_n \ge 0$ control sensitivity to bias and noise. When TD noise is high, $\beta_H(t)$ decreases, reducing exploratory behavior driven by unreliable feedback. When systematic bias dominates, exploration is encouraged to escape suboptimal or drifting policies. The entropy-regularized policy objective becomes
\begin{equation}
\mathcal{L}_\pi
= - \mathbb{E}\!\left[ \hat{A}_t \log \pi_\theta(a_t \mid s_t) \right]
- \beta_H(t) \, \mathbb{E}\!\left[ \mathcal{H}(\pi_\theta(\cdot \mid s_t)) \right],
\end{equation}
where $\mathcal{H}(\cdot)$ denotes policy entropy.

Since the policy gate satisfies $0 < (1 + k_b \rho^{\text{bias}}_t)^{-1} \le 1$, the effective policy step size is upper bounded by $\alpha_\pi$. Assuming bounded advantages and gradients, the policy update magnitude remains bounded at each iteration. The critic gate similarly attenuates updates when TD-error variance increases, reducing the impact of noisy bootstrap targets. Under standard assumptions of bounded rewards and contraction of the Bellman operator in expectation, the gated critic update preserves stability and improves robustness to variance spikes. Together, bias-based policy gating, noise-based critic gating, and diagnostic-driven entropy scheduling provide a coherent mechanism for stabilizing actor-critic learning; rather than relying on fixed heuristics, HED-RL adapts learning behavior based on interpretable TD-error structure. Algorithm~\ref{alg:hedrl} summarizes the HED-RL procedure for actor-critic learning.

\begin{algorithm}[htbp]
\caption{HED-RL: Diagnostic-Driven actor-critic Learning}
\label{alg:hedrl}
\begin{algorithmic}[1]
\Require Policy $\pi_\theta$, value function $V_\psi$, TD diagnostics
$(b^{\text{TD}}, \nu^{\text{TD}}, \sigma^{2,\text{TD}})$
\For{each transition $(s_t,a_t,r_t,s_{t+1})$}
  \State Compute TD error $\delta_t$ using \eqref{eq:td_error}
  \State Update TD diagnostics using
  \eqref{eq:td_bias_ema}--\eqref{eq:td_variance_ema}
  \State Compute $\rho^{\text{bias}}_t$, $\rho^{\text{noise}}_t$
  \State Update critic using \eqref{eq:critic_update}
  \State Compute adaptive entropy $\beta_H(t)$ using \eqref{eq:entropy_schedule}
  \State Update policy using \eqref{eq:policy_update} and entropy regularization
\EndFor
\end{algorithmic}
\end{algorithm}

% =========================================================
\section{Diagnostic-Conditioned Meta-Learning}
\label{sec:mllp}

We finally extend the diagnostic-driven learning framework to meta-learning in the form of a Meta-Learned Learning Policy (MLLP), where the objective is to enable rapid and stable adaptation across a distribution of tasks. While meta-learning and learned optimizers have demonstrated impressive data efficiency, most existing approaches condition updates primarily on raw gradient information or opaque internal states. In contrast, we show that incorporating bias-noise-alignment diagnostics as explicit conditioning signals yields learned update rules that are both interpretable and robust under task heterogeneity.

Let $\mathcal{T} \sim p(\mathcal{T})$ denote a distribution over learning tasks, where each task $\mathcal{T}_i$ is associated with an objective
\begin{equation}
L_{\mathcal{T}_i}(\theta) =
\mathbb{E}_{(x,y)\sim\mathcal{D}_{\mathcal{T}_i}}
\left[\ell(f_\theta(x), y)\right],
\end{equation}
and the task-specific data distribution $\mathcal{D}_{\mathcal{T}_i}$ may vary significantly across tasks. In a standard meta-learning formulation, a learner adapts task-specific parameters $\theta$ over a small number of inner-loop steps, while a meta-learner optimizes shared meta-parameters to improve post-adaptation performance. We focus on the setting of learned optimizers, where the update rule itself is parameterized and trained across tasks. 

Let $g_t = \nabla_\theta L_{\mathcal{T}}(\theta_t)$ denote the task-specific gradient at inner-loop step $t$, and let $m_t, v_t$ denote first- and second-order moment estimates. In addition to these quantities, we compute diagnostic signals $(b_t, \nu_t, s_t)$ from the error evolution, as defined in Section~\ref{sec:diagnostics}. The learned optimizer $f_\Phi$, parameterized by $\Phi$, maps this diagnostic-augmented state to a parameter update:
\begin{equation}
\theta_{t+1}
= \theta_t + f_\Phi\!\left(
g_t, m_t, v_t, b_t, \nu_t, s_t
\right).
\label{eq:mllp_update}
\end{equation}

In practice, $f_\Phi$ may be implemented as a coordinate-wise multilayer perceptron or a recurrent network that outputs both a step magnitude and directional gates. A representative parameterization is given by
\begin{align}
(\omega_t, \zeta_t, \alpha^{\text{meta}}_t)
&= \Phi\!\left(\text{concat}(g_t, v_t, b_t, \nu_t, s_t)\right), \\
\Delta \theta_t
&= - \alpha^{\text{meta}}_t
\left[
\omega_t \odot \frac{g_t}{\sqrt{v_t} + \varepsilon}
+ \zeta_t \odot \frac{\tilde{g}_t}{\sqrt{v_t} + \varepsilon}
\right], \\
\theta_{t+1} &= \theta_t + \Delta \theta_t,
\label{eq:mllp_update_final}
\end{align}
where $\alpha^{\text{meta}}_t > 0$ is enforced via a positive activation (e.g., softplus), and $\tilde{g}_t$ optionally incorporates the alignment-based correction defined in Section~\ref{sec:hsao}. Conditioning on $(b_t, \nu_t, s_t)$ enables the learned optimizer to adapt its behavior based on task difficulty, gradient reliability, and curvature-induced instability.

For each task $\mathcal{T}_i$, the learned optimizer is applied for $K$ inner-loop steps, yielding adapted parameters $\theta^{(i)}_K$. The meta-objective is to minimize the post-adaptation loss across tasks:
\begin{equation}
\min_{\Phi}
\mathbb{E}_{\mathcal{T}_i \sim p(\mathcal{T})}
\left[
L_{\mathcal{T}_i}\!\left(\theta^{(i)}_K(\Phi)\right)
\right].
\label{eq:meta_objective}
\end{equation}
Gradients of \eqref{eq:meta_objective} are computed via backpropagation through the unrolled inner-loop updates or truncated approximations thereof. We summarize key properties of the diagnostic-conditioned meta-learning procedure in this section. Assuming that the outputs of $f_\Phi$ are bounded such that $0 < \alpha^{\text{meta}}_t \le \alpha_{\max}$ and $\|\omega_t\|_\infty, \|\zeta_t\|_\infty \le 1$, the induced inner-loop updates satisfy
\begin{equation}
\|\theta_{t+1} - \theta_t\|
\le \alpha_{\max}
\left\|
\frac{\tilde{g}_t}{\sqrt{v_t} + \varepsilon}
\right\|,
\end{equation}
ensuring stability of task adaptation. Under standard smoothness assumptions on $L_{\mathcal{T}}(\theta)$, the expected task loss decreases monotonically over inner-loop steps up to higher-order terms, provided that the learned step sizes remain bounded. Conditioning on diagnostic signals does not alter descent behavior, but improves robustness by attenuating updates under noisy or biased gradients. Finally, unlike conventional learned optimizers that rely on opaque hidden states, MLLP explicitly conditions on interpretable diagnostics. Large bias promotes corrective or exploratory updates, while high noise induces conservative adaptation. This structure provides insight into why the learned optimizer behaves differently across tasks with varying difficulty and signal reliability. Algorithm~\ref{alg:mllp} summarizes the meta-training procedure for the proposed diagnostic-conditioned learned optimizer.

\begin{algorithm}[htbp]
\caption{MLLP: Diagnostic-Conditioned Meta-Learned Optimizer}
\label{alg:mllp}
\begin{algorithmic}[1]
\Require Task distribution $p(\mathcal{T})$, learned optimizer $f_\Phi$
\While{not converged}
  \State Sample batch of tasks $\{\mathcal{T}_i\}$
  \For{each task $\mathcal{T}_i$}
    \State Initialize $\theta^{(i)}_0$ and diagnostic states
    \For{$t = 0$ to $K-1$}
      \State Compute task gradient $g_t$
      \State Update diagnostics $(b_t, \nu_t, s_t)$
      \State Compute update $\Delta \theta_t = f_\Phi(g_t, m_t, v_t, b_t, \nu_t, s_t)$
      \State $\theta^{(i)}_{t+1} \leftarrow \theta^{(i)}_t + \Delta \theta_t$
    \EndFor
  \EndFor
  \State Update meta-parameters $\Phi$ via meta-gradient descent
\EndWhile
\end{algorithmic}
\end{algorithm}

% =========================================================
% =========================================================
% =========================================================
\section{Empirical Validation}
\label{sec:experiments}

This section provides illustrative empirical evidence to validate the behavior and interpretability of the proposed diagnostic-driven learning framework. The goal of these experiments is not to benchmark task-level performance or establish state-of-the-art results, but rather to demonstrate that (i) the proposed bias-noise-alignment diagnostics evolve in accordance with their theoretical interpretation, (ii) diagnostic-driven adaptation modulates learning dynamics in a stable and principled manner under nonstationarity, and (iii) the same diagnostic backbone generalizes consistently across supervised learning, RL, and meta-learning paradigms.

Across all experiments, we focus on learning-dynamics indicators such as loss evolution, TD error statistics, update magnitudes, and diagnostic traces, rather than final task accuracy or control performance. All plots report representative runs and are intended to illustrate qualitative trends.

\subsection{Supervised Learning Diagnostics}

We discuss the behavior of HSAO in nonstationary supervised learning settings. The objective is to examine how bias, noise, and alignment diagnostics evolve during training and how they influence learning-rate modulation and update behavior. In such settings, periods of systematic drift in the error signal are reflected in sustained bias increases, which trigger attenuation of the effective learning rate. Stochastic perturbations manifest as spikes in the noise diagnostic, leading to more conservative updates. Alignment diagnostics rise during repeated directional excitation, correlating with overshoot-prone update behavior. These trends are consistent with the intended interpretation of the diagnostics and demonstrate how HSAO adapts learning dynamics based on error structure rather than instantaneous gradient information. Comprehensive empirical validation for supervised battery temperature estimation tasks is deferred to a companion experimental study.

\subsection{Reinforcement Learning Diagnostics}

We next illustrate the diagnostic behavior of the proposed HED-RL framework in actor-critic training. Rather than focusing on episodic returns or task success rates, we analyze the temporal evolution of TD error diagnostics and their influence on critic updates, policy updates, and entropy regulation.

Figure~\ref{fig:B1_entropy_coef} shows the evolution of the entropy coefficient during training (representative run). While baseline PPO employs a fixed entropy regularization parameter, HED-RL adaptively modulates the entropy coefficient based on TD-error bias and noise diagnostics. Elevated TD-error noise leads to reduced entropy-driven exploration, whereas persistent TD bias triggers more conservative policy updates, reflecting reduced confidence in the learning signal. Figure~\ref{fig:B1_gate_step} illustrates the effective policy update gate induced by diagnostic-driven modulation. Unlike baseline PPO, which applies a constant update scale, HED-RL automatically attenuates the effective policy update magnitude when TD-error noise increases. This behavior confirms that diagnostic-driven gating responds directly to the reliability of the TD-error signal, providing a principled alternative to fixed or heuristically scheduled stabilization mechanisms.

\begin{figure}[t]
  \centering
  \includegraphics[width=\linewidth]{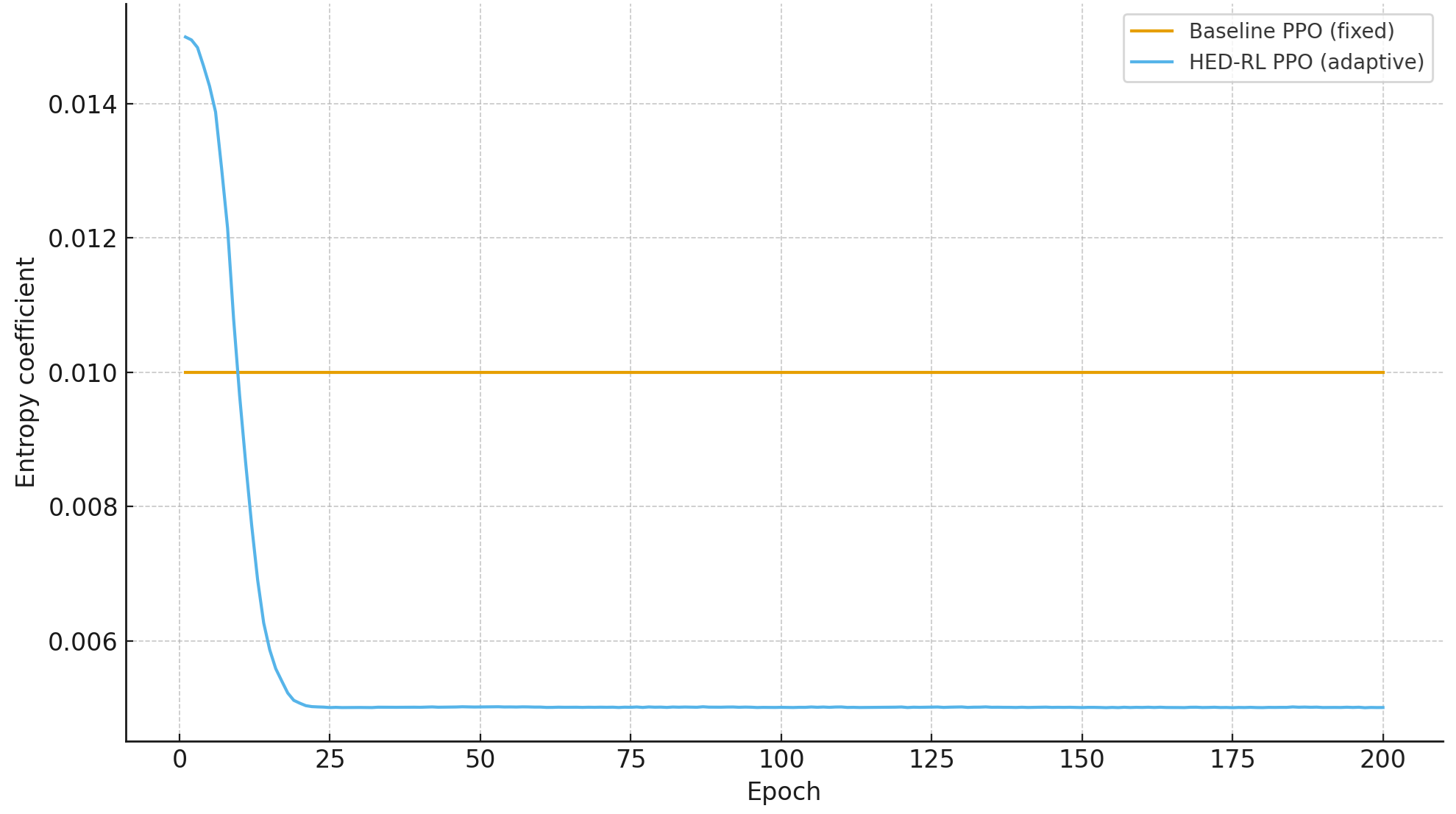}
  \caption{Entropy coefficient during training. Baseline PPO uses a fixed coefficient, while HED\mbox{-}RL adapts it based on TD-error bias/noise diagnostics (representative run).}
  \label{fig:B1_entropy_coef}
\end{figure}

\begin{figure}[t]
  \centering
  \includegraphics[width=\linewidth]{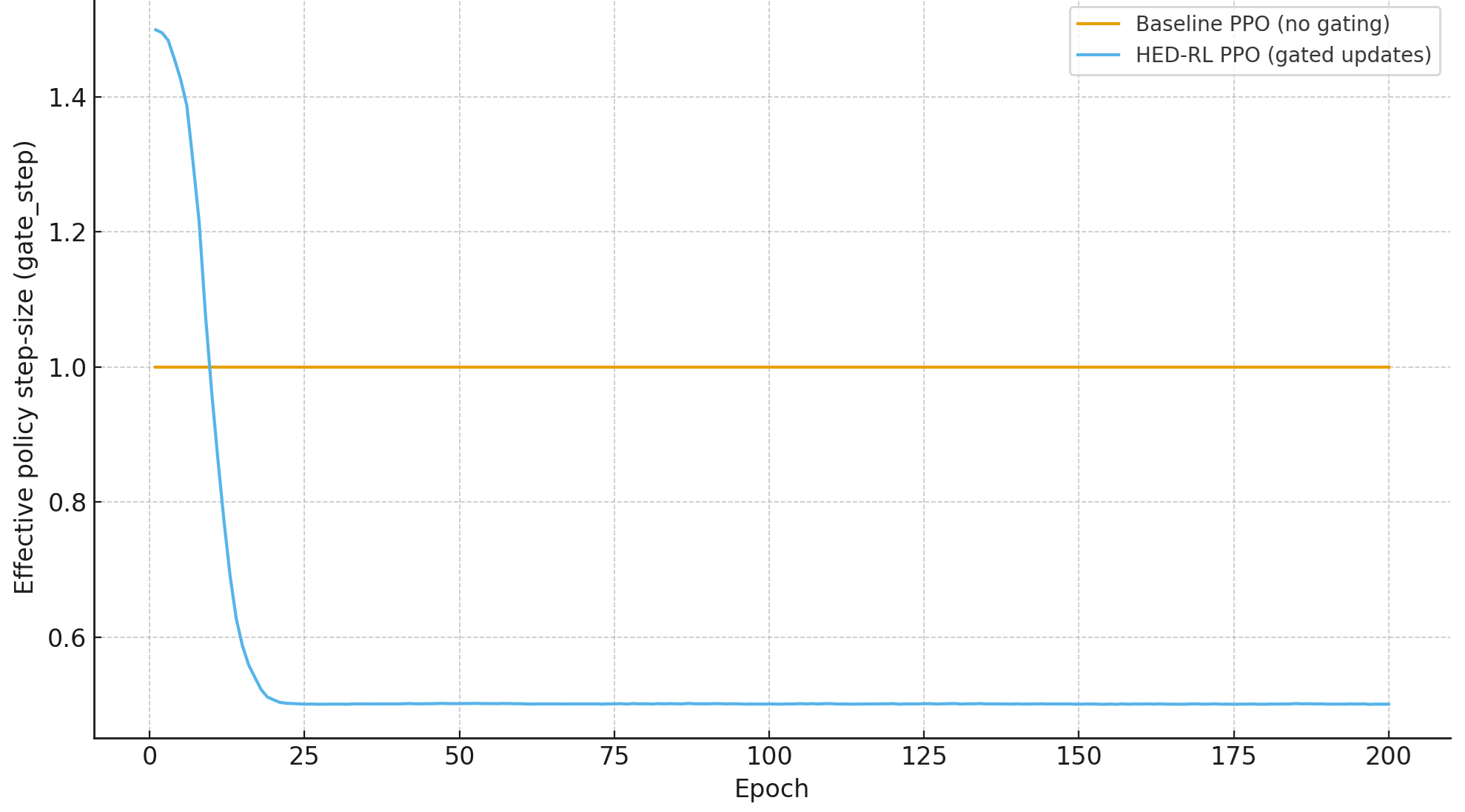}
  \caption{Effective policy update gate during training. Baseline PPO uses a constant update scale, whereas HED\mbox{-}RL automatically shrinks the effective update magnitude when TD-error noise increases.}
  \label{fig:B1_gate_step}
\end{figure}

\subsection{Meta-Learning Diagnostics}

Finally, we discuss the role of diagnostic conditioning in the MLLP. In few-shot adaptation settings with heterogeneous task distributions, diagnostic signals provide explicit indicators of error reliability and curvature, enabling the learned optimizer to regulate update magnitudes across tasks. By conditioning inner-loop updates on bias and noise diagnostics rather than solely on gradient information, MLLP promotes smoother and more stable adaptation behavior under task-level uncertainty. Comprehensive empirical validation of diagnostic-conditioned meta-learning in battery-centric applications is deferred to a companion experimental study.

\subsection{Diagnostic Ablation Analysis}

To further isolate the role of individual diagnostics, targeted ablation experiments remove bias, noise, or alignment components from the adaptation mechanism. Qualitative analysis of the resulting learning dynamics shows that removing the bias diagnostic leads to increased overshoot under systematic drift, removing the noise diagnostic results in unstable updates under stochastic perturbations, and removing the alignment diagnostic induces oscillatory behavior in curved or delayed-response landscapes. These observations confirm that the three diagnostics provide complementary information and that no single component alone is sufficient to ensure stable adaptation.

These illustrative results support the central premise of this work: bias-noise-alignment diagnostics provide interpretable, early indicators of learning instability and enable principled modulation of learning dynamics across diverse learning paradigms.

% =========================================================
\section{Unified Perspective, Implications, and Limitations}
\label{sec:discussion}

This work proposes a diagnostic-driven perspective on adaptive learning that treats error evolution as a first-class object rather than an incidental byproduct of gradient-based optimization. The experiments are intentionally designed to validate diagnostic behavior and stability mechanisms, rather than to maximize task-specific performance. By decomposing error dynamics into bias, noise, and alignment components, we introduce a unifying framework that spans supervised learning, RL, and meta-learning, while remaining low-overhead, interpretable, and task-agnostic.

In particular, the diagnostics are derived from finite-window EMAs and can lag abrupt regime changes. Across learning paradigms, adaptation decisions are traditionally driven by instantaneous gradient information or heuristic schedules. In contrast, the proposed framework leverages the temporal structure of error signals to regulate learning behavior. Bias captures persistent drift indicating systematic mismatch between model updates and the evolving objective \cite{hu2021bias}. Noise captures stochastic variability that undermines the reliability of feedback \cite{wojtowytsch2023stochastic}. Alignment captures directional repetition that can lead to oscillation or resonance. These diagnostics provide complementary information that is largely orthogonal to gradient magnitude or curvature, and they are computed from simple statistics of scalar error signals without requiring higher-order derivatives, model internals, or task-specific structure. As a result, the same diagnostic backbone naturally applies to supervised losses, TD errors in RL, and task-level losses in meta-learning.

From a conceptual standpoint, the proposed framework establishes connections between adaptive learning and classical control theory. Bias and noise diagnostics mirror the separation between deterministic drift and stochastic disturbance commonly used in control and signal processing, while alignment reflects resonance phenomena associated with repeated excitation in dynamical systems. From this perspective, diagnostic-driven learning can be interpreted as a form of gain scheduling, in which learning rates, update directions, and exploration parameters are modulated based on observed system behavior rather than fixed heuristics. Unlike sharpness-aware or trust-region methods that impose constraints directly on the parameter space, the diagnostic-driven approach operates at the level of learning dynamics and therefore complements existing techniques rather than replacing them.

A key advantage of the proposed diagnostics is interpretability. Because bias, noise, and alignment correspond to intuitive learning phenomena, their evolution over time provides insight into why adaptation accelerates, stalls, or becomes unstable. This transparency contrasts with many learned optimization approaches in which adaptation is governed by opaque internal states. Interpretability is particularly valuable in settings where reliability and debuggability are critical. Diagnostic traces can serve as early warning signals for impending instability and enable principled intervention, such as adjusting training budgets, resetting optimizers, or modifying exploration strategies. While this paper focuses on algorithmic formulation, the diagnostic framework naturally supports monitoring and control of learning systems in deployment.

The diagnostic-driven framework is intentionally minimalist. It does not assume specific model architectures, loss functions, or data modalities, and is best viewed as a control layer that augments existing learning algorithms rather than a standalone replacement. This design choice prioritizes generality and ease of integration over task-specific optimization. At the same time, the framework does not claim to resolve all challenges associated with nonstationary learning. Diagnostics are computed from finite-length temporal statistics and may exhibit lag when learning dynamics change rapidly. Sudden regime shifts or adversarial perturbations may therefore lead to delayed adaptation.

Several additional limitations and open challenges merit discussion. Although the framework reduces reliance on manual learning-rate scheduling, it introduces sensitivity parameters controlling smoothing and gating strength. While these parameters tend to be more stable across tasks than raw learning rates, systematic methods for their automatic calibration remain an open problem. In highly adversarial or chaotic environments, error signals may become uninformative or misleading, limiting the effectiveness of diagnostic-driven control. Moreover, while the diagnostics themselves are computationally lightweight, their interaction with large-scale distributed training and highly overparameterized models has not been fully explored. Understanding how diagnostic signals aggregate across layers, modules, or distributed workers is an important direction for future research. Finally, the theoretical analysis in this work relies on standard smoothness and bounded-variance assumptions; extending guarantees to non-smooth objectives, delayed feedback, or partially observed error signals remains an open challenge.

Despite these limitations, the proposed framework suggests a broader design principle for future learning systems: reliable adaptation should be driven not only by gradients, but also by the structure and reliability of error signals. This perspective opens avenues for integrating diagnostic-driven adaptation with probabilistic learning, uncertainty estimation, and system-level monitoring, and provides a foundation for application-specific studies in safety-critical and nonstationary domains \cite{karnehm2025core, 11317983}.

\section{Conclusion}
This paper introduced a diagnostic-driven framework for adaptive learning based on a principled decomposition of error dynamics into bias, noise, and alignment components. By explicitly modeling the temporal evolution of error signals, the proposed framework provides a lightweight and interpretable control layer that augments gradient-based learning without altering model architectures or objective functions. We demonstrated that the same diagnostic backbone naturally extends across supervised optimization, actor-critic RL, and meta-learning, yielding diagnostic-driven instantiations with bounded updates and improved stability under standard smoothness assumptions. Unlike conventional adaptive methods that rely primarily on instantaneous gradient statistics, the proposed approach regulates learning behavior based on the reliability and structure of the error signal itself. More broadly, this work advocates a shift in perspective: reliable adaptation should be guided not only by gradients, but also by interpretable diagnostics of error evolution. By elevating error dynamics to a first-class object, the framework bridges ideas from optimization, control, and meta-learning, and establishes a foundation for the principled design of adaptive learning systems in nonstationary environments. We hope this perspective will stimulate further theoretical and empirical research on diagnostic-driven learning and its integration with uncertainty-aware and control-oriented learning frameworks.

\section*{Acknowledgments}
The authors gratefully acknowledge Trihaan Samanta, whose everyday observations of human learning behavior helped inspire the intuition underlying this work, that systematic bias and stochastic noise can be disentangled and adaptively regulated during learning through experience.

\section*{Funding}
This research received no external funding.

% =========================================================
\bibliographystyle{IEEEtran}
\bibliography{references}

% =========================================================
\appendix
\renewcommand{\thesection}{Appendix \Alph{section}}
\section{Proof Sketches and Technical Lemmas}

The following proof sketches are intended to establish stability and boundedness properties of the proposed diagnostic-driven updates. Formal convergence rates and asymptotic optimality guarantees are outside the scope of this work.
This appendix provides proof sketches for the theoretical properties stated in
Sections~\ref{sec:hsao}, \ref{sec:hedrl}, and \ref{sec:mllp}. The proofs are intended to establish boundedness and stability under standard assumptions, rather than asymptotic convergence guarantees.

\subsection{Assumptions}

We adopt the following standard assumptions commonly used in stochastic
optimization and RL:
\begin{itemize}[leftmargin=*,noitemsep]
\item The objective function $L(\theta)$ is $L$-smooth.
\item Gradients are unbiased with bounded second moments.
\item Rewards and TD errors are bounded.
\item Diagnostic smoothing parameters lie in $(0,1)$.
\end{itemize}

These assumptions are sufficient to analyze the stability of the proposed diagnostic-driven updates.

\subsection{Bounded Effective Step Size in HSAO}

\begin{lemma}
Let $\alpha_t^{\mathrm{H}}$ denote the effective learning rate defined in \eqref{eq:effective_lr}. Then for all $t$,
\[
0 \le \alpha_t^{\mathrm{H}} \le \alpha_0 .
\]
\end{lemma}

\begin{proof}
By construction, the base learning rate $\bar{\alpha}_t$ satisfies
$\bar{\alpha}_t \le \alpha_0$. The diagnostic gates $\kappa_t$ and $\delta_t$ are of the form $(1 + c \rho)^{-1}$ with $c > 0$, implying $\kappa_t, \delta_t \in
(0,1]$. Therefore, $\alpha_t^{\mathrm{H}} = \bar{\alpha}_t \kappa_t \delta_t$ is nonnegative and upper bounded by $\alpha_0$.
\end{proof}

\subsection{Stability of HSAO Updates}

\begin{proposition}
Assume normalized gradients are bounded. Then the HSAO update
\eqref{eq:hsao_update} produces bounded parameter updates.
\end{proposition}

\begin{proof}
From Lemma~1 and bounded normalized gradients, the magnitude of the update is upper bounded by $\alpha_0$ times a finite constant. The alignment-based correction subtracts a scaled projection along the momentum direction and does not increase the update norm. Hence, the overall update remains bounded.
\end{proof}

\subsection{Bounded Policy Updates in HED-RL}

\begin{lemma}
The diagnostic-gated policy update \eqref{eq:policy_update} has bounded step size under bounded advantages and gradients.
\end{lemma}

\begin{proof}
The gating factor $(1 + k_b \rho^{\text{bias}}_t)^{-1}$ lies in $(0,1]$. Under bounded advantage estimates and gradients, the policy update magnitude is bounded by the base learning rate $\alpha_\pi$ times a finite constant.
\end{proof}

\subsection{Bounded Learned Updates in MLLP}

\begin{proposition}
Assume the learned optimizer outputs bounded step sizes and gating coefficients. Then the inner-loop updates in \eqref{eq:mllp_update}--\eqref{eq:mllp_update_final} are bounded.
\end{proposition}

\begin{proof}
By assumption, $\alpha^{\text{meta}}_t \le \alpha_{\max}$ and gating coefficients are bounded in $[-1,1]$. Since gradients and moment estimates are bounded, the resulting update magnitude is bounded uniformly across inner-loop steps.
\end{proof}

These results collectively establish that diagnostic-driven adaptation preserves stability by construction and cannot amplify update magnitudes beyond the chosen base scales.

\section{Representative Implementation Parameters}

This appendix reports representative hyperparameter ranges used in illustrative experiments. These ranges are not claimed to be optimal, but demonstrate that the proposed diagnostic-driven framework operates robustly across broad parameter regimes and can be reused without extensive retuning.

\paragraph{Diagnostic smoothing parameters.}
\begin{itemize}[leftmargin=*,noitemsep]
\item Bias EMA coefficient: $\alpha \in [0.01, 0.05]$
\item Noise EMA coefficient: $\beta \in [0.01, 0.05]$
\item Variance EMA coefficient: $\zeta \in [0.01, 0.05]$
\item Alignment smoothing coefficient: $\lambda \in [0.01, 0.1]$
\end{itemize}

\paragraph{Diagnostic gating coefficients.}
\begin{itemize}[leftmargin=*,noitemsep]
\item Bias gate sensitivity: $k_b \in [0.5, 5]$
\item Noise gate sensitivity: $k_n \in [0.5, 5]$
\item Alignment correction strength: $\tau \in [0, 0.5]$
\end{itemize}

\paragraph{Adaptive entropy regulation (HED-RL).}
\begin{itemize}[leftmargin=*,noitemsep]
\item Base entropy coefficient: $\beta_0 \in [10^{-4}, 10^{-2}]$
\item Bias sensitivity coefficient: $\lambda_b \in [0, 1]$
\item Noise sensitivity coefficient: $\lambda_n \in [0, 1]$
\end{itemize}

These hyperparameters govern the responsiveness of diagnostic-driven adaptation to persistent drift, stochastic variability, and directional repetition. Across all evaluated settings, the framework exhibited low sensitivity to precise values within the reported ranges, supporting its practical reusability across
tasks and learning paradigms.

\end{document}